\documentclass{article}
\usepackage{fullpage}

\usepackage{graphicx}  
\usepackage{amssymb}
\usepackage{amsmath}  
\usepackage{amstext}  
\usepackage{amsthm}   
\usepackage{natbib}
\usepackage[flushleft]{threeparttable}  
\usepackage{float}  
\usepackage{hyperref}  
\newtheorem{definition}{Definition}
\usepackage{algorithm,algorithmic}
\newtheorem{thm}{Theorem}

\newtheorem{lemma}{Lemma}

\newtheorem{ass}[thm]{Assumption}

\hypersetup{colorlinks = true, citecolor = blue}  
\usepackage{subfigure}

\def \E {\mathrm{E}}

\def \u {\mathbf{u}}

\def \R {\mathbb{R}}

\def \a {\mathbf{a}}
\def \b {\mathbf{b}}

\def \P {\mathbb{P}}

\def \E {\mathrm{E}}

\def \R {\mathbb{R}}
\def \u {\mathbf{u}}

\def \x {\mathbf{x}}
\def \X {\mathcal{X}}
\def \Y {\mathcal{Y}}
\def \y {\mathbf{y}}
\def \W {\mathcal{W}}
\def \w {\mathbf{w}}
\def \yh {\widehat{\y}}
\def \yT {\y^{\text{LS}}}

\begin{document}
\title{\bf Towards Understanding Label Smoothing}
\author{Yi Xu, Yuanhong Xu, Qi Qian, Hao Li, Rong Jin\\ 
Machine Intelligence Technology, Alibaba Group\\
\{yixu, yuanhong.xuyh, qi.qian, lihao.lh, jinrong.jr\}@alibaba-inc.com
}
\date{First Version: June 15, 2020\\
Second Version: October 2, 2020} 
\maketitle

\begin{abstract}
Label smoothing regularization (LSR) has a great success in training  deep neural networks by stochastic algorithms such as stochastic gradient descent and its variants. However, the theoretical understanding of its power from the view of optimization is still rare. This study opens the door to a deep understanding of LSR by initiating the analysis. In this paper, we analyze the convergence behaviors of stochastic gradient descent with label smoothing regularization for solving non-convex problems and show that an appropriate LSR can help to speed up the convergence by reducing the variance. More interestingly, we proposed a simple 
yet effective strategy, namely {\bf T}wo-{\bf S}tage {\bf LA}bel smoothing algorithm (TSLA), that uses LSR in the early training epochs and drops it off in the later training epochs. We observe from the improved convergence result of TSLA that it benefits from LSR in the first stage and essentially converges faster in the second stage. To the best of our knowledge, this is the first work for understanding the power of LSR via establishing convergence complexity of stochastic methods with LSR in non-convex optimization.  We empirically demonstrate the effectiveness of the proposed method in comparison with baselines on training ResNet models over benchmark data sets.
\end{abstract}

\section{Introduction}
In training deep neural networks, one common strategy is to minimize cross-entropy loss with one-hot label vectors, which may lead to overfitting during the training progress that would lower the generalization accuracy~\citep{muller2019does}. To overcome the overfitting issue, several regularization techniques such as $\ell_1$-norm or $\ell_2$-norm penalty over the model weights, Dropout which randomly sets the outputs of neurons to zero~\citep{hinton2012improving}, batch normalization~\citep{ioffe2015batch}, and data augmentation~\citep{simard1998transformation}, are employed to prevent the deep learning models from becoming over-confident. However, these regularization techniques conduct on the hidden activations or weights of a neural network. As an output regularizer, label smoothing regularization (LSR)~\citep{szegedy2016rethinking} is proposed to improve the generalization and learning efficiency of a neural network by replacing the one-hot vector labels with the smoothed labels that average the hard targets and the uniform distribution of other labels. Specifically, for a $K$-class classification problem, the one-hot label is smoothed by $\yT = (1-\theta)\y + \theta \yh$, where $\y$ is the one-hot label, $\theta\in(0,1)$ is the smoothing strength and $\yh = \frac{{\bf1}}{K}$ is a uniform distribution for all labels. Extensive experimental results have shown that LSR has significant successes in many deep learning applications including image classification~\citep{zoph2018learning, he2019bag}, speech recognition~\citep{chorowski2017towards, zeyer2018improved}, and language translation~\citep{vaswani2017attention, nguyen2019transformers}. 

Due to the importance of LSR, researchers try to explore its behavior in training deep neural networks. \cite{muller2019does}~have empirically shown that the LSR can help improve model calibration, however, they also have found that LSR could impair knowledge distillation, that is, if one trains a teacher model with LSR, then a student model has worse performance. \cite{yuan2019revisit} have proved that LSR provides a virtual teacher model for knowledge distillation. 
As a widely used trick, \cite{lukasik2020does} have shown that LSR works since it can successfully mitigate label noise.  However, to the best of our knowledge, it is unclear, at least from a theoretical viewpoint, how the introduction of label smoothing will help improve the training of deep learning models, and to what stage, it can help. In this paper, we aim to provide an affirmative answer to this question and try to deeply understand why and how the LSR works from the view of optimization. Our theoretical analysis will show that an appropriate LSR can essentially reduce the variance of stochastic gradient in the assigned class labels and thus it can speed up the convergence.  Moreover, we will propose a novel strategy of employing LSR that tells when to use LSR. We summarize the main contributions of this paper as follows.
\begin{itemize}
\item It is the {\bf first work} that establishes improved iteration complexities of stochastic gradient descent (SGD)~\citep{robbins1951stochastic} with LSR for finding an $\epsilon$-approximate stationary point~(Definition~\ref{def:stationary}) in solving a smooth non-convex problem in the presence of an appropriate label smoothing. The results theoretically explain why an appropriate LSR can help speed up the convergence. (Section~\ref{sec:LSR})
\item We propose a simple yet effective strategy, namely {\bf T}wo-{\bf S}tage {\bf LA}bel smoothing (TSLA) algorithm, where in the first stage it trains models for certain epochs using a stochastic method with LSR while in the second stage it runs the same stochastic method without LSR. The proposed TSLA is a generic strategy that can incorporate many stochastic algorithms. With an appropriate label smoothing, we show that TSLA integrated with SGD has an {\bf improved} iteration complexity, compared to the SGD with LSR and the SGD without LSR. (Section~\ref{sec:TSLA})
\end{itemize}
\section{Related Work}
In this section, we introduce some related work. A closely related idea to LSR is confidence penalty proposed by~\cite{pereyra2017regularizing}, an output regularizer that penalizes confident output distributions by adding its negative entropy to the negative log-likelihood during the training process. The authors~\citep{pereyra2017regularizing} presented extensive experimental results in training deep neural networks to demonstrate better generalization comparing to baselines with only focusing on the existing hyper-parameters. They have shown that LSR is equivalent to confidence penalty with a reversing direction of KL divergence between uniform distributions and the output distributions. 

DisturbLabel introduced by~\cite{xie2016disturblabel} imposes the regularization within the loss layer, where it randomly replaces some of the ground truth labels as incorrect values at each training iteration. Its effect is quite similar to LSR that can help to prevent the neural network training from overfitting. The authors have verified the effectiveness of DisturbLabel via several experiments on training image classification tasks.  

Recently, many works~\citep{zhang2018mixup, bagherinezhad2018label, goibert2019adversarial, shen2019defending, li2020colam} explored the idea of LSR technique. \cite{ding2019adaptive} extended an adaptive label regularization method, which enables the neural network to use both correctness and incorrectness during training. \cite{pang2018towards} used the reverse cross-entropy loss to smooth the classifier's gradients. 
\cite{wang2020inference} proposed a graduated label smoothing method that uses the higher smoothing penalty for high-confidence predictions than that for low-confidence predictions. They found that the proposed method can improve both inference calibration and translation performance for neural machine translation models. By contrast, in this paper, we will try to understand the power of LSR from an optimization perspective and try to study how and when to use LSR. 

\section{Preliminaries and Notations}
We first present some notations. Let $\nabla_\w F(\w)$ denote the gradient of a function $F(\w)$. When the variable to be taken a gradient is obvious, we use $\nabla F(\w)$ for simplicity. We use $\|\cdot\|$ to denote the Euclidean norm. Let $\langle \cdot,\cdot\rangle$ be the inner product.

In classification problem, we aim to seek a classifier to map an example $\x\in\X$ onto one of $K$ labels $\y\in\Y \subset \R^K$, where $\y=(y_1,y_2,\dots,y_K)$ is a one-hot label, meaning that $y_i$ is ``1" for the correct class
and ``0" for the rest. Suppose the example-label pairs are draw from a distribution $\P$, i.e., $(\x,\y) \sim \P = (\P_\x, \P_\y)$. 
we denote by $\E_{(\x,\y)}[\cdot]$ the expectation that takes over a random variable $(\x,\y)$. When the randomness is obvious, we write $\E[\cdot]$ for simplicity. Our goal is to learn a prediction function $f(\w;\x): \W\times\X \to \R^K$ that is as close as possible to $\y$, where $\w\in\W$ is the parameter and $\W$ is a closed convex set. To this end, we want to minimize the following expected loss under $\P$:
\begin{align}\label{opt:prob}
   \min_{\w\in\W} F(\w) := \E_{
   (\x,\y)}\left[\ell(\y, f(\w;\x)) \right],
\end{align}
where $\ell: \Y \times \R^K \to \R_+$ is a cross-entropy loss function given by
\begin{align}\label{loss:CE}
   \ell(\y,f(\w;\x)) = \sum_{i=1}^{K} -y_i\log\left(\frac{\exp(f_i(\w;\x))}{\sum_{j=1}^{K}\exp(f_j(\w;\x))}\right).
\end{align}
The objective function $F(\w)$ is not convex since $f(\w;\x)$ is non-convex in terms of $\w$. To solve the problem~(\ref{opt:prob}), one can simply use some iterative methods such as stochastic gradient descent (SGD). Specifically, at each training iteration $t$, SGD updates solutions iteratively by
\begin{align*}
    \w_{t+1} = \w_t - \eta \nabla_\w \ell(\y_t, f(\w_t;\x_t)),
\end{align*}
where $\eta>0$ is a learning rate. 

Next, we present some notations and assumptions that will be used in the convergence analysis. Throughout this paper, we also make the following assumptions for solving the problem (\ref{opt:prob}).
\begin{ass}\label{ass:2}
Assume the following conditions hold: 
\begin{itemize}
\item[(i)]  The stochastic gradient of $F(\w)$ is unbiased, i.e., $\E_{(\x,\y)}[\nabla \ell(\y, f(\w;\x))] = \nabla F(\w)$, and the variance of stochastic gradient is bounded, i.e., there exists a constant $\sigma^2>0$, such that $$\E_{(\x,\y)}\left[\left\|\nabla \ell(\y, f(\w;\x)) - \nabla F(\w)\right\|^2\right] = \sigma^2.$$ 
\item[(ii)] $F(\w)$ is smooth with an $L$-Lipchitz continuous gradient, i.e., it is differentiable and there exists a constant $L>0$ such that $$\|\nabla F(\w)  - \nabla F(\u)\|\leq L\|\w - \u\| ,\forall \w, \u \in\W.$$
\end{itemize}
\end{ass}
{\bf Remark.} Assumption~\ref{ass:2} (i) and (ii) are commonly used assumptions in the literature of non-convex optimization~\citep{ghadimi2013stochastic,yangnonconvexmo, yuan2019stagewise, wang2019spiderboost, li2020exponential}. 
Assumption~\ref{ass:2} (ii) says the objective function is $L$-smooth, and it has an equivalent expression~\citep{opac-b1104789} which is $$F(\w) - F(\u) \le \langle \nabla F(\u), \w - \u \rangle + \frac{L}{2}\|\w-\u\|^2, \forall \w, \u \in \W.$$   

Let $\yh$ be a label introduced for smoothing label. Then the smoothed label $\yT$ is given by
\begin{align}\label{label:smooth}
\yT = (1 - \theta) \y + \theta \yh,
\end{align}
where $\theta \in (0, 1)$ is the smoothing strength, $\y$ is the one-hot label.  
Similar to label $\y$, we suppose the label $\yh$ is draw from a distribution $\P_{\yh}$. We introduce the variance of stochastic gradient using label $\yh$ as follows. 
\begin{align}\label{variance:output:sm:label}
    \E_{(\x,\yh)}\left[\left\|\nabla \ell(\yh, f(\w;\x)) - \nabla F(\w)\right\|^2\right] = \widehat \sigma^2 := \delta\sigma^2.
\end{align}
where $\delta > 0$ is a constant and $\sigma^2$ is defined in Assumption~\ref{ass:2} (i). We make several remarks for (\ref{variance:output:sm:label}).\\
{\bf Remark.} (a) We do not require the stochastic gradient $\nabla \ell(\yh, f(\w;\x))$ is unbiased, i.e., it could be $\E[\nabla \ell(\yh, f(\w;\x))] \neq \nabla F(\w)$. (b) The variance $\widehat\sigma^2$ is defined based on the label $\yh$ rather than the smoothed label $\yT$. (c) We do not assume the variance $\widehat\sigma^2$ is bounded since $\delta$ could be an arbitrary value, however, we will discuss the different cases of $\delta$ in our analysis. If $\delta\ge 1$, then $\widehat\sigma^2 \ge \sigma^2$; while if $0<\delta<1$, then $\widehat\sigma^2 < \sigma^2$. It is worth mentioning that $\delta$ could be small when an appropriate label is used in the label smoothing. For example, one can smooth labels by using a teacher model~\citep{hinton2015distilling} or the model's own distribution~\citep{reed2014training}. In the first paper of label smoothing~\citep{szegedy2016rethinking} and the following related studies~\citep{muller2019does,yuan2019revisit}, researchers consider a uniform distribution over all $K$ classes of labels as the label $\yh$, i.e., set $\yh = \frac{{\bf 1}}{K}$. 

We now introduce an important assumption regarding $F(\w)$, i.e. there is no very bad local optimum on the surface of objective function $F(\w)$. More specifically, the following assumption holds.
\begin{ass}\label{ass:3}
There exists a constant $\mu>0$ such that $2\mu(F(\w)- F_*) \le \|\nabla F(\w)\|^2, \forall \w\in\W$, where $F_* = \min_{\w\in\W} F(\w)$ is the optimal value.
\end{ass}
{\bf Remark.} This property is known as Polyak-\L ojasiewicz (PL) condition~\citep{polyak1963gradient}, and it has been theoretically and empirically observed in training deep neural networks~\citep{allen2019convergence, yuan2019stagewise}. This condition is widely used to establish convergence in the literature of non-convex optimization, please see~\citep{yuan2019stagewise, wang2019spiderboost, karimi2016linear, li2018simple, charles2018stability, li2020exponential} and references therein. 

To measure the convergence of non-convex and smooth optimization problems as in~\citep{nesterov1998introductory, ghadimi2013stochastic,yangnonconvexmo}, we need the following definition of the first-order stationary point.
\begin{definition}[First-order stationary point]\label{def:stationary} For the problem of $\min_{\w\in\W}F(\w)$, a point $\w\in\W$ is called a first-order stationary
point if $\|\nabla f(\w)\| = 0$. Moreover, if $\|\nabla f(\w)\| \leq \epsilon$, then the point $\w$ is said to be an $\epsilon$-stationary point, where $\epsilon\in(0,1)$ is a small positive value.
\end{definition}

\section{Convergence Analysis of SGD with LSR}\label{sec:LSR}
To understand LSR from the optimization perspective, we consider SGD with LSR in Algorithm~\ref{alg:lsr} for the sake of simplicity. The only difference between Algorithm~\ref{alg:lsr} and standard SGD is the use of the output label for constructing a stochastic gradient. The following theorem shows that Algorithm~\ref{alg:lsr} converges to an approximate stationary point in expectation under some conditions. We include its proof in Appendix~\ref{app:thm:lsr}.
\begin{algorithm}[t]
\caption{SGD with Label Smoothing Regularization}\label{alg:lsr}
\begin{algorithmic}[1]
\STATE \textbf{Initialize}:  $\w_0\in\W$,  $\theta\in(0,1)$, set $\eta$ as the value in Theorem~\ref{thm:lsr}.
\FOR{$t=0,1,\ldots,T-1$}
    \STATE sample $(\x_t,\y_t)$, set $\yT_t = (1 - \theta) \y_t + \theta \yh_t$
	\STATE update $\w_{t+1} = \w_t - \eta\nabla_\w \ell(\yT_t, f(\w_t;\x_t))$
\ENDFOR
\end{algorithmic}
\end{algorithm}
\begin{thm}\label{thm:lsr}
Under Assumption \ref{ass:2}, run Algorithm~\ref{alg:lsr} with $\eta=\frac{1}{L}$ and $\theta=\frac{1}{1+\delta}$, then 
$\E_R[\|\nabla F(\w_R)\|^2] \le  \frac{2F(\w_{0}) }{\eta T} + 2 \delta \sigma^2$,
where $R$ is uniformly sampled from $\{0, 1, \dots, T-1\}$.  Furthermore, we have the following two results. \\
(1) when $ \delta  \le \frac{\epsilon^2}{4 \sigma^2} $, if we set $T = \frac{4F(\w_{0}) }{\eta \epsilon^2} $, then Algorithm~\ref{alg:lsr} converges to an $\epsilon$-stationary point in expectation, i.e., $\E_R[\|\nabla F(\w_R)\|^2]\le\epsilon^2$. The total sample complexity is $T= O\left(\frac{1}{\epsilon^2}\right)$. \\
(2) when $ \delta  > \frac{\epsilon^2}{4 \sigma^2}$, if we set $T = \frac{F(\w_{0}) }{ \eta  \delta \sigma^2 }$, then Algorithm~\ref{alg:lsr} does not converge to an $\epsilon$-stationary point, but we have $\E_R[\|\nabla F(\w_R)\|^2]\le 4 \delta \sigma^2 \le O(\delta)$.
\end{thm} 
{\bf Remark.} We observe that the variance term is $2 \delta \sigma^2$, instead of $\eta L\sigma^2 $ for standard analysis of SGD without LSR (i.e., $\theta=0$, please see the detailed analysis of Theorem~\ref{thm:0} in Appendix~\ref{supp:baseline}). 
For the convergence analysis, the different between SGD with LSR and SGD without LSR is that $\nabla \ell(\yh, f(\w;\x)) $ is not an unbiased estimator of $\nabla F(\w)$ when using LSR.
The convergence behavior of Algorithm~\ref{alg:lsr} heavily depends on the parameter $\delta$. When $\delta$ is small enough, say $\delta \le O(\epsilon^2)$ with a small positive value $\epsilon\in(0,1)$, then Algorithm~\ref{alg:lsr} converges to an $\epsilon$-stationary point with the total sample complexity of $O\left(\frac{1}{\epsilon^2}\right)$. Recall that the total sample complexity of standard SGD without LSR for finding an $\epsilon$-stationary point is $O\left( \frac{1}{\epsilon^4}\right)$ (\citep{DBLP:journals/mp/GhadimiL16,ghadimi2016mini}, please also see the detailed analysis of Theorem~\ref{thm:0} in Appendix~\ref{supp:baseline}). 
The convergence result shows that if we could find a label $\yh$ that has a reasonably small amount of $\delta$, we will be able to reduce sample complexity for training a learning model from $O\left( \frac{1}{\epsilon^4}\right)$ to $O\left( \frac{1}{\epsilon^2}\right)$. Thus, the reduction in variance will happen when an appropriate label smoothing with $\delta \in(0,1)$ is introduced. We will find in the empirical evaluations that different label $\yh$ lead to different performances and an appropriate selection of label $\yh$ has a better performance (see the performances of LSR and LSR-pre in Table~\ref{table2}).
On the other hand, when the parameter $\delta$ is large such that $\delta > \Omega(\epsilon^2)$, that is to say, if an inappropriate label smoothing is used, then Algorithm~\ref{alg:lsr} does not converge to an $\epsilon$-stationary point, but it converges to a worse level of $O(\delta)$. 

\section{TSLA: A Generic Two-Stage Label Smoothing Algorithm}\label{sec:TSLA}
Despite superior outcomes in training  deep neural networks, some real applications have shown the adverse effect of LSR. \cite{muller2019does}~have empirically observed that LSR impairs distillation, that is, after training teacher models with LSR, student models perform worse. The authors believed that LSR reduces mutual information between input example and output logit. \cite{kornblith2019better} have found that LSR impairs the accuracy of transfer learning when training deep neural network models on ImageNet data set. \cite{seo2020self} trained deep neural network models for few-shot learning on miniImageNet and found a significant performance drop with LSR. This motivates us to investigate a strategy that combines the algorithm with and without LSR during the training progress. 
Let think in this way, one possible scenario is that training one-hot label is ``easier" than training smoothed label. Taking the cross entropy loss in (\ref{loss:CE}) for an example, one need to optimize a single loss function $-\log\left(\exp(f_k(\w;\x))/\sum_{j=1}^{K}\exp(f_j(\w;\x))\right)$ when one-hot label (e.g, $y_k=1$ and $y_i=0$ for all $i\ne k$) is used, but need to optimize all $K$ loss functions $-\sum_{i=1}^{K} \yT_i\log\left(\exp(f_i(\w;\x))/\sum_{j=1}^{K}\exp(f_j(\w;\x))\right)$ when smoothed label (e.g., $\yT = (1-\theta) \y+\theta\frac{{\bf 1}}{K}$ so that $y^{\text{LS}}_k=1 - (K-1)\theta/K$ and $y^{\text{LS}}_i=\theta/K$ for all $i\ne k$) is used. Nevertheless, training deep neural networks is gradually focusing on hard examples with the increase of training epochs. It seems that training smoothed label in the late epochs makes the learning progress more difficult. 
In addition, after LSR, we focus on optimizing the overall distribution that contains the minor classes, which are probably not important at the end of training progress.
One question is whether LSR helps at the early training epochs but it has less (even negative) effect during the later training epochs? This question encourages us to propose and analyze a simple strategy with LSR dropping that switches a stochastic algorithm with LSR to the algorithm without LSR. 

\begin{algorithm}[t]
\caption{The TSLA algorithm}\label{alg:tsla}
\begin{algorithmic}[1]
\STATE \textbf{Initialize}:  $\w_0\in\W$, $\theta\in(0,1)$, $\eta_1, \eta_2>0$
\STATE  \textbf{Input: } stochastic algorithm $\mathcal A$ (e.g., SGD)\\
{\color{blue}// First stage: $\mathcal A$ with LSR}
\FOR{$t=0,1,\ldots,T_1-1$}
	\STATE sample $(\x_t,\y_t)$, set $\yT_t = (1 - \theta) \y_t + \theta \yh_t$
	\STATE update $\w_{t+1} = \mathcal A$-step$(\w_t; \x_t, \yT_t, \eta_1)$ \hfill{$\diamond$  one update step of $\mathcal A$}
\ENDFOR\\
{\color{blue}// Second stage: $\mathcal A$ without LSR}
\FOR{$t=T_1,1,\ldots,T_1+T_2-1$}
    \STATE sample $(\x_t,\y_t)$
	\STATE update $\w_{t+1} = \mathcal A$-step$(\w_t; \x_t, \y_t, \eta_2)$ \hfill{$\diamond$  one update step of $\mathcal A$} 
\ENDFOR
\end{algorithmic}
\end{algorithm}

In this subsection, we propose a generic framework that consists of two stages, wherein the first stage it runs a stochastic algorithm $\mathcal A$ (e.g., SGD) with LSR in $T_1$ iterations and the second stage it runs the same algorithm without LSR up to $T_2$ iterations. This framework is referred to as {\bf T}wo-{\bf S}tage {\bf LA}bel smoothing (TSLA) algorithm, whose updating details are presented in Algorithm~\ref{alg:tsla}. The notation $\mathcal A$-step$(\cdot;\cdot,\eta)$ is one update step of a stochastic algorithm $\mathcal A$ with learning rate $\eta$. For example, if we select SGD as algorithm~$\mathcal A$, then
\begin{align}\label{sdg:step}
&\text{SGD-step}(\w_t; \x_t, \yT_t, \eta_1) = \w_t - \eta_1  \nabla \ell(\yT_t, f(\w_t;\x_t)),\\
&\text{SGD-step}(\w_t; \x_t, \y_t, \eta_2) = \w_t - \eta_2 \nabla \ell(\y_t, f(\w_t;\x_t)).
\end{align}
The proposed TSLA is a generic strategy where the subroutine algorithm $\mathcal A$ can be replaced by any stochastic algorithms such as momentum SGD~\citep{polyak1964some}, Stochastic Nesterov's Accelerated Gradient~\citep{citeulike9501961}, and adaptive algorithms including {\sc AdaGrad}~\citep{duchi2011adaptive}, RMSProp~\citep{hinton2012neural}, AdaDelta~\citep{zeiler2012adadelta}, Adam~\citep{kingma2015adam}, Nadam~\citep{dozat2016incorporating} and {\sc AMSGrad}~\citep{reddi2018convergence}. Please note that the algorithm can use different learning rates $\eta_1$ and $\eta_2$ during the two stages. The last solution of the first stage will be used as the initial solution of the second stage. If $T_1 = 0$, then TSLA reduces to the baseline, i.e., a standard stochastic algorithm $\mathcal A$ without LSR; while if $T_2 = 0$, TSLA becomes to LSR method, i.e., a standard stochastic algorithm $\mathcal A$ with LSR.

\subsection{Convergence Result of TSLA}
In this subsection, we will give the convergence result of the proposed TSLA algorithm. For simplicity, we use SGD as the subroutine algorithm $\mathcal A$ in the analysis. The convergence result in the following theorem shows the power of LSR from the optimization perspective. Its proof is presented in Appendix~\ref{app:thm:drop}. It is easy to see from the proof that by using the last output of the first stage as the initial point of the second stage, TSLA can enjoy the advantage of LSR in the second stage with an improved convergence.
\begin{thm}\label{thm:drop}
Under Assumptions \ref{ass:2}, \ref{ass:3}, suppose $\sigma^2 \delta/\mu  \le  F(\w_0)$, run Algorithm~\ref{alg:tsla} with $\mathcal A$ = SGD, $\theta=\frac{1}{1+\delta}$, $\eta_1 = \frac{1}{L}$, $T_1=\log\left( \frac{2\mu F(\w_0)(1+\delta)}{2\delta\sigma^2  }\right)/(\eta_1\mu)$, $\eta_2 = \frac{\epsilon^2}{2L\sigma^2}$ and $T_2 = \frac{8 \delta \sigma^2 }{\mu \eta_2\epsilon^2}$, 
then $\E_R[\|\nabla F(\w_R)\|^2]\le\epsilon^2$, where $R$ is uniformly sampled from $\{T_1, \dots, T_1+T_2-1\}$.
\end{thm} 
{\bf Remark.} 
\begin{table}[t]
\caption{Comparisons of Total Sample Complexity}\label{table0}
\centering
  \def\sym#1{\ifmmode^{#1}\else\(^{#1}\)\fi}
  \begin{tabular}{l*{3}{l}}
\hline
   Condition on $\delta$ &  TSLA & LSR & baseline  \\
\hline
    $\Omega(\epsilon^2)<\delta  $  &  $\frac{\delta}{\epsilon^4}$ & $\infty$    &  $\frac{1}{\epsilon^4}$  \\
\hline
    $\delta = O(\epsilon^2)$  &  $\frac{1}{\epsilon^2}$ & $\frac{1}{\epsilon^2}$    &  $\frac{1}{\epsilon^4}$  \\
\hline
    $\Omega(\epsilon^4)<\delta < O(\epsilon^2)$  &  $\frac{1}{\epsilon^{2-\theta}}^*$ & $\frac{1}{\epsilon^2}$    &  $\frac{1}{\epsilon^4}$  \\
\hline
    $\Omega(\epsilon^{4+c})\le \delta \le O(\epsilon^4)^{**}$  &  $
\log\left(\frac{1}{\epsilon}\right)$ & $\frac{1}{\epsilon^2}$    &  $\frac{1}{\epsilon^4}$  \\
\hline
{\small $^*\theta\in(0,2)$; $^{**} c \ge 0$ is a constant}
  \end{tabular}
\end{table} 
It is obvious that the learning rate $\eta_2$ in the second stage is roughly smaller than the learning rate $\eta_1$ in the first stage, which matches the widely used stage-wise learning rate decay scheme in training neural networks. To explore the total sample complexity of TSLA, we consider different conditions on $\delta$. We summarize the total sample complexities of finding $\epsilon$-stationary points for SGD with TSLA (TSLA), SGD with LSR (LSR), and SGD without LSR (baseline) in Table~\ref{table0}, where $\epsilon\in (0,1)$ is the target convergence level, and we only present the orders of the complexities but ignore all constants. When $\Omega(\epsilon^2)<\delta <1$, LSR dose not converge to an $\epsilon$-stationary point (denoted by $\infty$), while TSLA reduces sample complexity from $O\left( \frac{1}{\epsilon^4}\right)$ to $O\left( \frac{\delta}{\epsilon^4}\right)$, compared to the baseline. When $\delta < O(\epsilon^2)$, the total complexity of TSLA is between $\log(1/\epsilon)$ and $1/\epsilon^2$, which is always better than LSR and the baseline. In summary, TSLA achieves the best total sample complexity by enjoying the good property of an appropriate label smoothing (i.e., when $0<\delta<1$). However, when $\delta \ge 1$, baseline has better convergence than TSLA, meaning that the selection of label $\yh$ is not appropriate.

\section{Experiments}
To further evaluate the performance of the proposed TSLA method, we trained deep neural networks on three benchmark data sets, CIFAR-100~\citep{krizhevsky2009learning}, Stanford Dogs~\citep{khosla2011novel} and CUB-2011~\citep{WahCUB_200_2011}, for image classification tasks. CIFAR-100~\footnote{\url{https://www.cs.toronto.edu/~kriz/cifar.html}} has 50,000 training images and 10,000 testing images of 32$\times$32 resolution with 100 classes.  Stanford Dogs data set~\footnote{\url{http://vision.stanford.edu/aditya86/ImageNetDogs/}} contains 20,580 images of 120 breeds of dogs, where 100 images from each breed is used for training. CUB-2011~\footnote{\url{http://www.vision.caltech.edu/visipedia/CUB-200.html}} is a birds image data set with 11,788 images of 200 birds species. The ResNet-18 model~\citep{he2016deep} is applied as the backbone in the experiments. We compare the proposed TSLA incorporating with SGD (TSLA) with two baselines, SGD with LSR (LSR) and SGD without LSR (baseline). The mini-batch size of training instances for all methods is $256$ as suggested by~\cite{he2019bag} and~\cite{he2016deep}. The momentum parameter is fixed as 0.9. 
\begin{table}[t]
\caption{Comparison of Testing Accuracy for Different Methods (mean $\pm$ standard deviation, in $\%$).}\label{table1}
\begin{center}
  \def\sym#1{\ifmmode^{#1}\else\(^{#1}\)\fi}
  \begin{tabular}{l*{4}{l}}
    \hline
   &  \multicolumn{2}{c}{Stanford Dogs} & \multicolumn{2}{c}{CUB-2011}  \\
     \cline{2-3}\cline{4-5} 
  Algorithm$^*$  & \multicolumn{1}{l}{Top-1 accuracy } & \multicolumn{1}{l}{Top-5 accuracy}
     & \multicolumn{1}{l}{Top-1 accuracy } & \multicolumn{1}{l}{Top-5 accuracy}  \\
    \hline
    baseline& 82.31 $\pm$ 0.18  &   97.76 $\pm$ 0.06 & 75.31 $\pm$ 0.25    &   93.14 $\pm$ 0.31    \\
    \hline
    LSR &    82.80 $\pm$ 0.07  &   97.41 $\pm$ 0.09  &  76.97 $\pm$ 0.19 &  92.73 $\pm$ 0.12     \\
    \hline
     TSLA(20)   &  83.15 $\pm$ 0.02  &  97.91 $\pm$ 0.08  &  76.62 $\pm$ 0.15  &  93.60 $\pm$ 0.18   \\
    TSLA(30)  &    83.89 $\pm$ 0.16  &  98.05 $\pm$ 0.08 &  77.44 $\pm$ 0.19  &  93.92 $\pm$ 0.16 \\
    TSLA(40)   &    {\bf 83.93} $\pm$ 0.13  &  98.03 $\pm$ 0.05   &  77.50 $\pm$ 0.20  &  93.99 $\pm$ 0.11  \\
    TSLA(50)   &    83.91 $\pm$ 0.15  &  {\bf 98.07} $\pm$ 0.06   &  {\bf 77.57} $\pm$ 0.21  &  93.86 $\pm$ 0.14   \\
    TSLA(60)   &  83.51 $\pm$ 0.11  &  97.99 $\pm$ 0.06   &  77.25 $\pm$ 0.29   &  {\bf 94.43} $\pm$ 0.18        \\
    TSLA(70)  &    83.38 $\pm$ 0.09  &  97.90 $\pm$ 0.09 &  77.21 $\pm$  0.15  &  93.31 $\pm$ 0.12   \\
    TSLA(80)   &    83.14 $\pm$ 0.09  &  97.73 $\pm$ 0.07   &  77.05 $\pm$  0.14   &  93.05 $\pm$  0.08       \\
    \hline
  \end{tabular}
  \end{center}
    $^*${\small TSLA($s$): TSLA drops off LSR after epoch $s$.}
\end{table}
\begin{figure}
    \centering
    \includegraphics[width=0.34\textwidth]{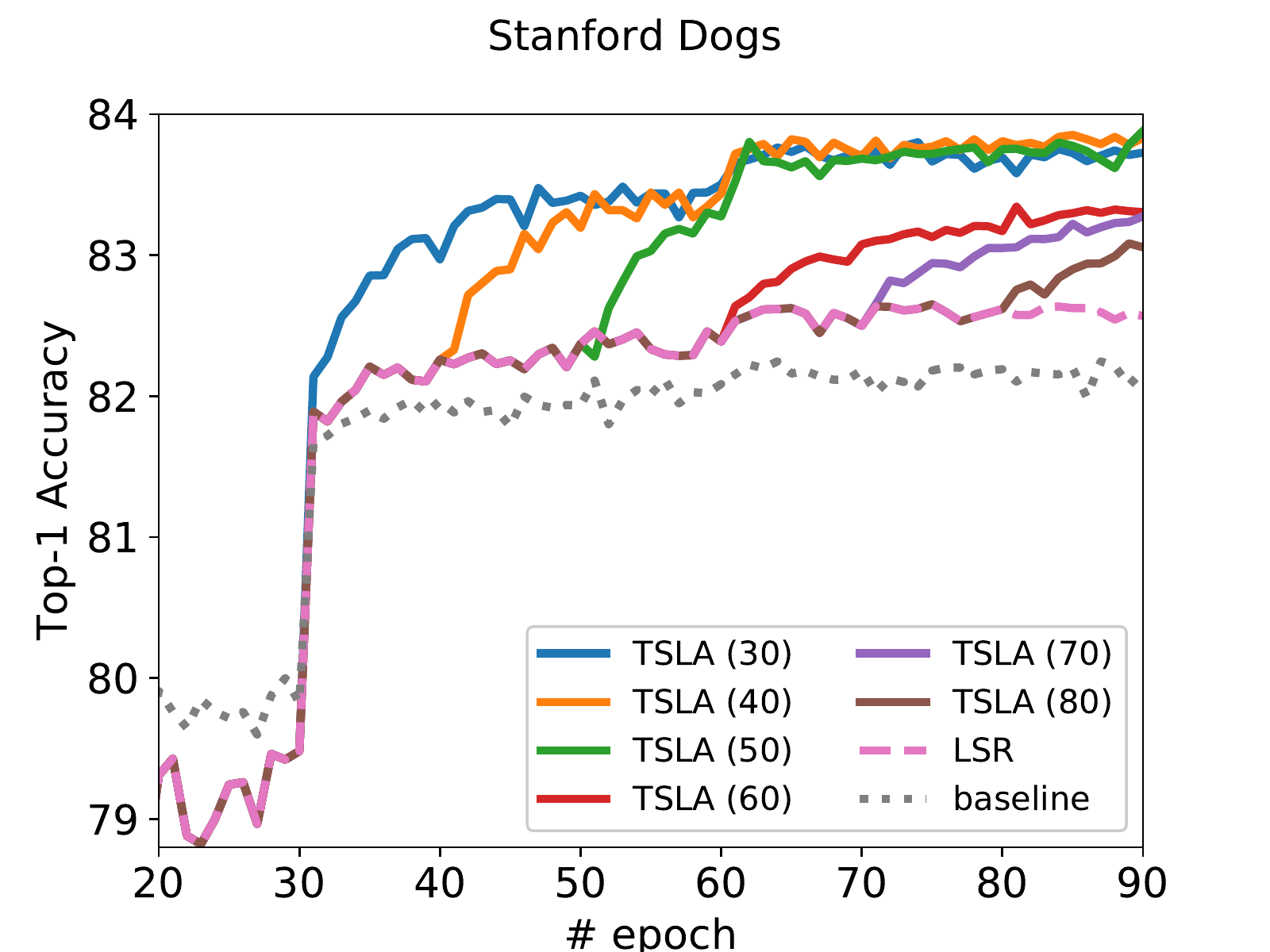}\hspace{-0.15in}
    \includegraphics[width=0.34\textwidth]{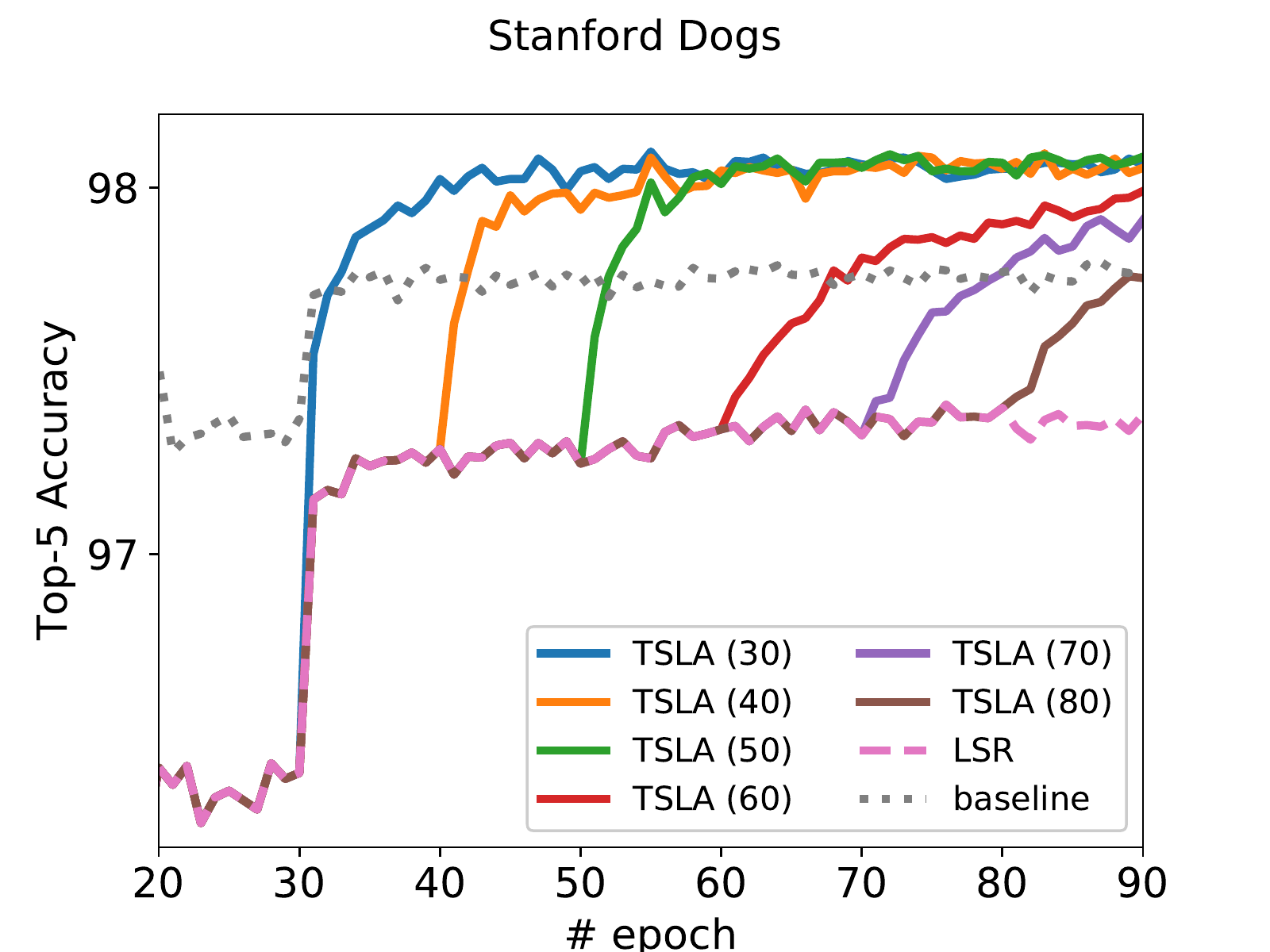}\hspace{-0.15in}
   \includegraphics[width=0.34\textwidth]{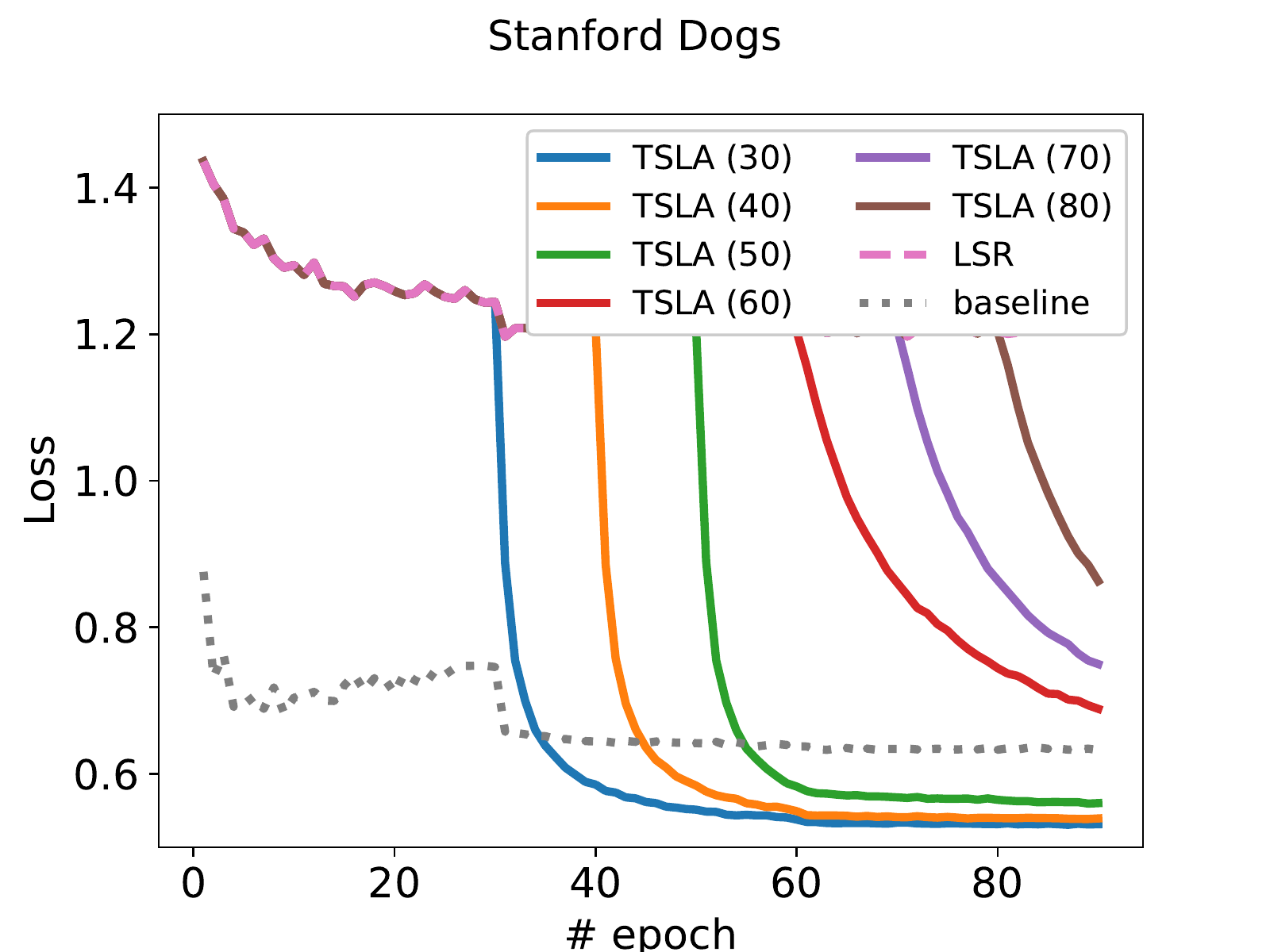}
        \includegraphics[width=0.34\textwidth]{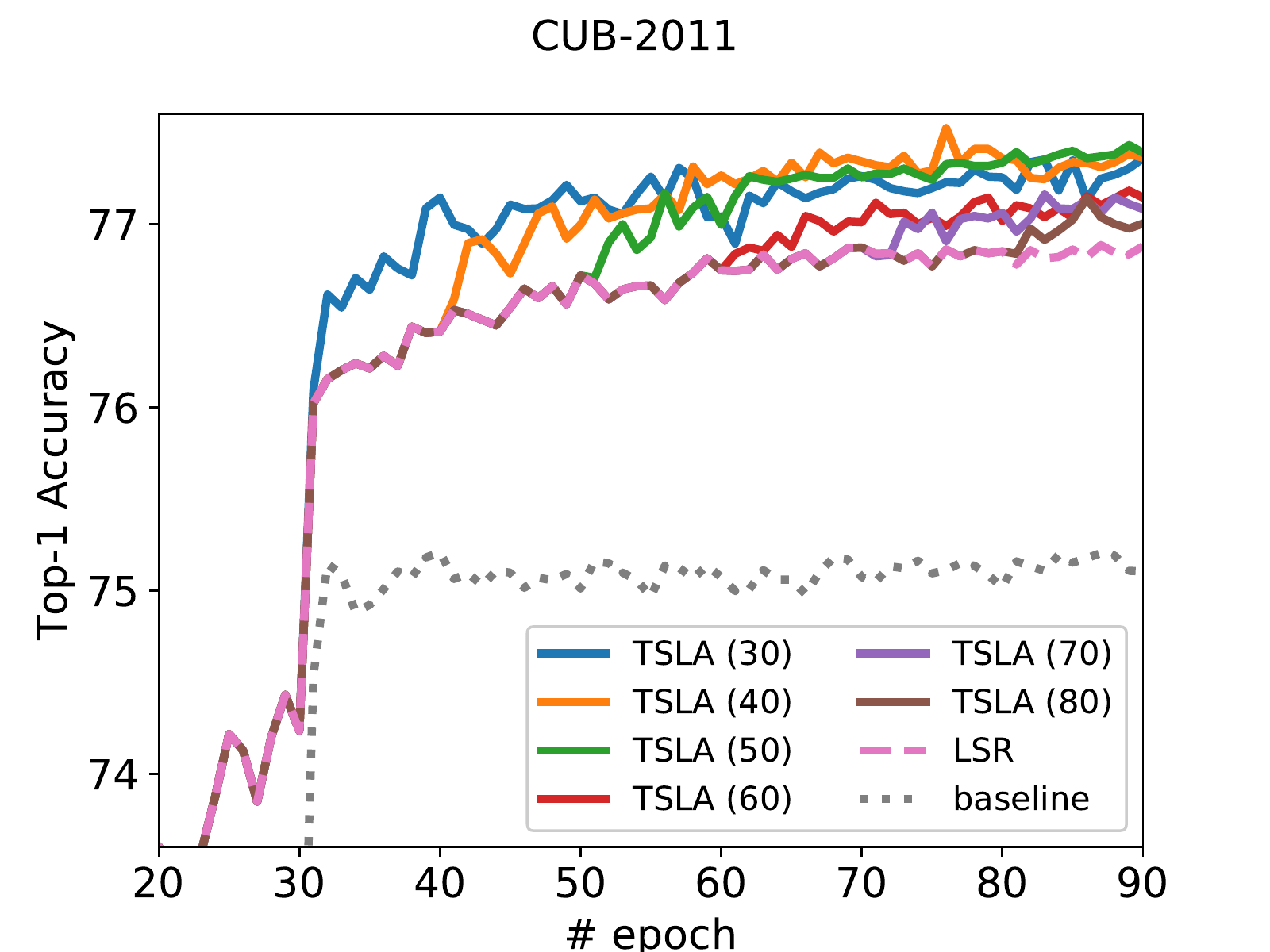}\hspace{-0.15in}
            \includegraphics[width=0.34\textwidth]{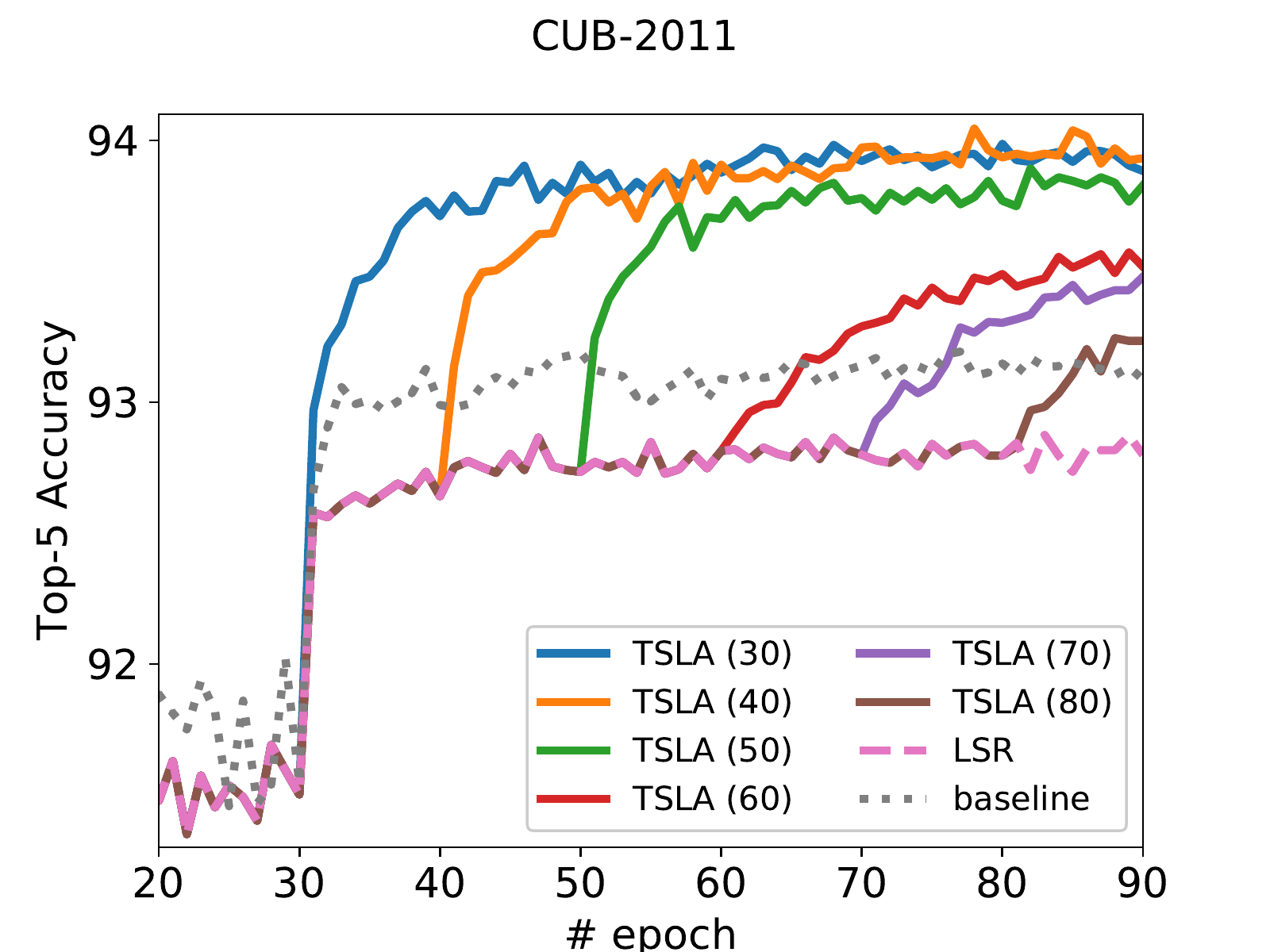}\hspace{-0.15in}
 \includegraphics[width=0.34\textwidth]{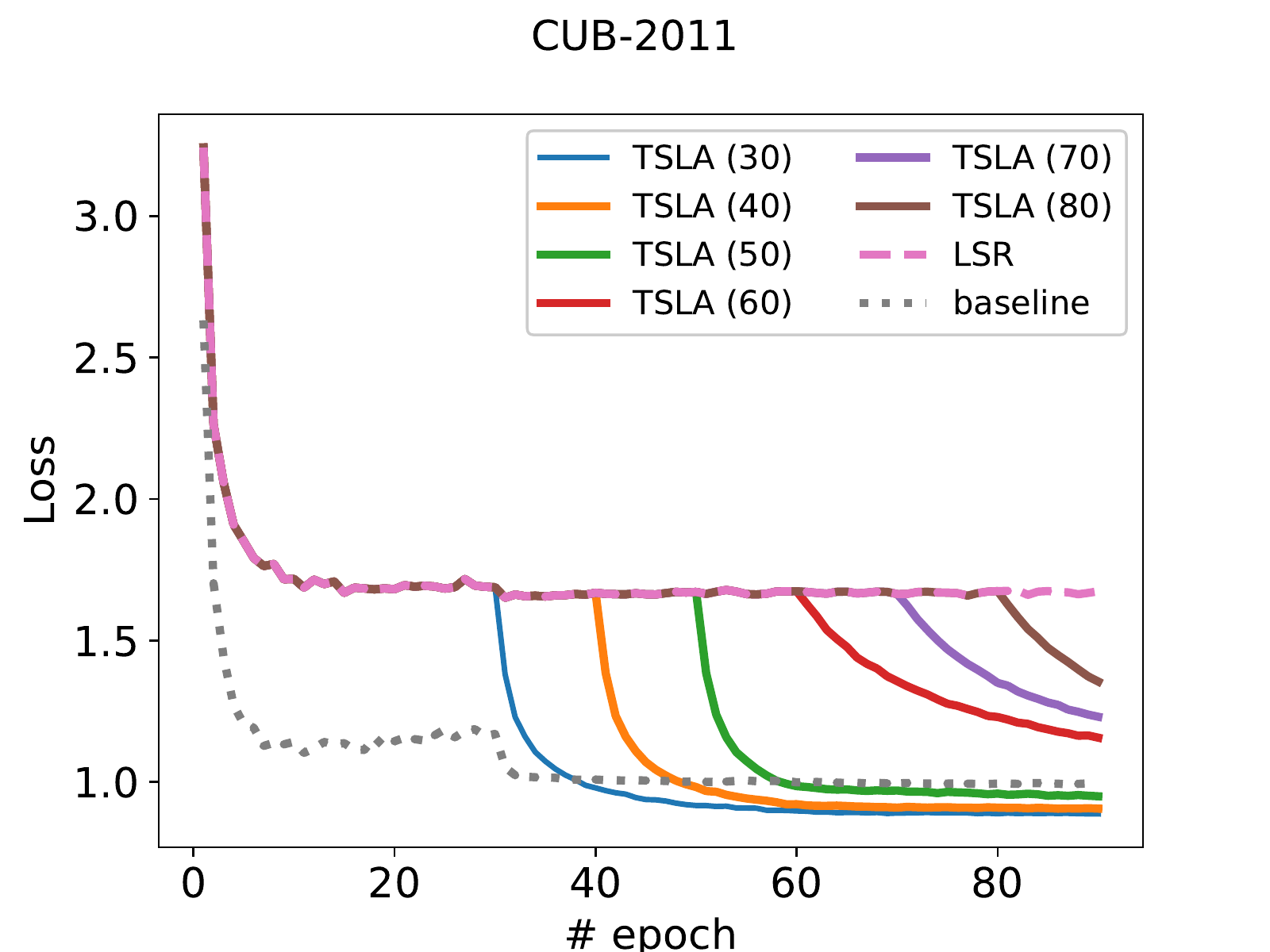}
    \caption{Testing Top-1, Top-5 Accuracy and Loss on ResNet-18 over Stanford Dogs and CUB-2011. TSLA($s$) means TSLA drops off LSR after epoch $s$.}
     \label{fig1}
\end{figure}
\subsection{Stanford Dogs and CUB-2011}
We separately train ResNet-18~\citep{he2016deep} up to 90 epochs over two data sets Stanford Dogs and CUB-2011. We use weight decay with the parameter value of $10^{-4}$. For all algorithms, the initial learning rates for FC are set to be 0.1, while that for the pre-trained backbones are 0.001 and 0.01 for Standford Dogs and CUB-2011, respectively. The learning rates are divided by 10 every 30 epochs. For LSR, we fix the value of smoothing strength $\theta=0.4$ for the best performance, and the 
label $\yh$ used for label smoothing is set to be a uniform distribution over all $K$ classes, i.e., $\yh=\frac{\bf{1}}{K}$. The same values of the smoothing strength $\theta$ and the same $\yh$ are used during the first stage of TSLA. For TSLA, we drop off the LSR (i.e., let $\theta=0$) after $s$ epochs during the training process, where $s \in\{20, 30, 40, 50, 60, 70, 80\}$. We first report the highest top-1 and top-5 accuracy on the testing data sets for different methods. All top-1 and top-5 accuracy are averaged over 5 independent random trails with their standard deviations. The results of the comparison are summarized in Table~\ref{table1}, where the notation ``TSLA($s$)" means that the TSLA algorithm drops off LSR after epoch $s$. It can be seen from Table~\ref{table1} that under an appropriate hyperparameter setting the models trained using TSLA outperform that trained using LSR and baseline, which supports the convergence result in Section~\ref{sec:TSLA}. We notice that the best top-1 accuracy of TSLA are TSLA(40) and TSLA(50) for Stanford Dogs and CUB-2011, respectively, meaning that the performance of TSLA($s$) is not monotonic over the dropping epoch $s$. For CUB-2011, the top-1 accuracy of TSLA(20) is smaller than that of LSR. This result matches the convergence analysis of TSLA showing that it can not drop off LSR too early. For top-5 accuracy, we found that TSLA(80) is slightly worse than baseline. This is because of dropping LSR too late so that the update iterations (i.e., $T_2$) in the second stage of TSLA is too small to converge to a good solution. We also observe that LSR is better than baseline regarding top-1 accuracy but the result is opposite as to top-5 accuracy. We then plot the averaged top-1 accuracy, averaged top-5 accuracy, and averaged loss among 5 trails of different methods in Figure~\ref{fig1}. We remove the results for TSLA(20) since it dropped off LSR too early as mentioned before. The figure shows TSLA improves the top-1 and top-5 testing accuracy immediately once it drops off LSR. Although TSLA may not converges if it drops off LSR too late, see TSLA(60), TSLA(70), and TSLA(80) from the third column of Figure~\ref{fig1}, it still has the best performance compared to LSR and baseline. TSLA(30), TSLA(40), and TSLA(50) can converge to lower objective levels, comparing to LSR and baseline.
\subsection{CIFAR-100}
\begin{table}[t]
\caption{Comparison of Testing Accuracy for Different Methods (mean $\pm$ standard deviation, in $\%$).}\label{table2}
\begin{center}
  \def\sym#1{\ifmmode^{#1}\else\(^{#1}\)\fi}
  \begin{tabular}{l*{4}{l}}
    \hline
    & \multicolumn{2}{c}{CIFAR-100}    \\
     \cline{2-3}
    Algorithm$^*$ & \multicolumn{1}{l}{Top-1 accuracy } & \multicolumn{1}{l}{Top-5 accuracy} \\
    \hline
    baseline&  76.87 $\pm$ 0.04    &   93.47 $\pm$ 0.15     \\
    \hline
    LSR&  77.77 $\pm$ 0.18 &  93.55 $\pm$ 0.11     \\
    \hline
    TSLA(120)   &  77.92 $\pm$ 0.21  &  94.13 $\pm$ 0.23       \\
    TSLA(140) &  77.93 $\pm$  0.19  &  94.11 $\pm$ 0.22   \\
    TSLA(160)   &  77.96 $\pm$  0.20   &  94.19 $\pm$  0.21       \\
    TSLA(180)  & 78.04 $\pm$ 0.27   &  94.23 $\pm$ 0.15   \\
    \hline
    LSR-pre &  78.07 $\pm$ 0.31 &  94.70 $\pm$ 0.14     \\
    \hline
    TSLA-pre(120)   &  78.34 $\pm$ 0.31  &  94.68 $\pm$ 0.14       \\
    TSLA-pre(140)  &  78.39 $\pm$  0.25  &  94.73 $\pm$ 0.11   \\
    TSLA-pre(160)   &  {\bf 78.55} $\pm$  0.28   &  94.83 $\pm$  0.08       \\
    TSLA-pre(180)  & 78.53 $\pm$ 0.23   &  {\bf 94.96} $\pm$ 0.23   \\
    \hline
  \end{tabular}
    \end{center}
    $^*${\small TSLA($s$)/TSLA-pre($s$): TSLA/TSLA-pre drops off LSR/LSR-pre after epoch $s$.}
\end{table}
\begin{figure}
    \centering
           \includegraphics[width=0.34\textwidth]{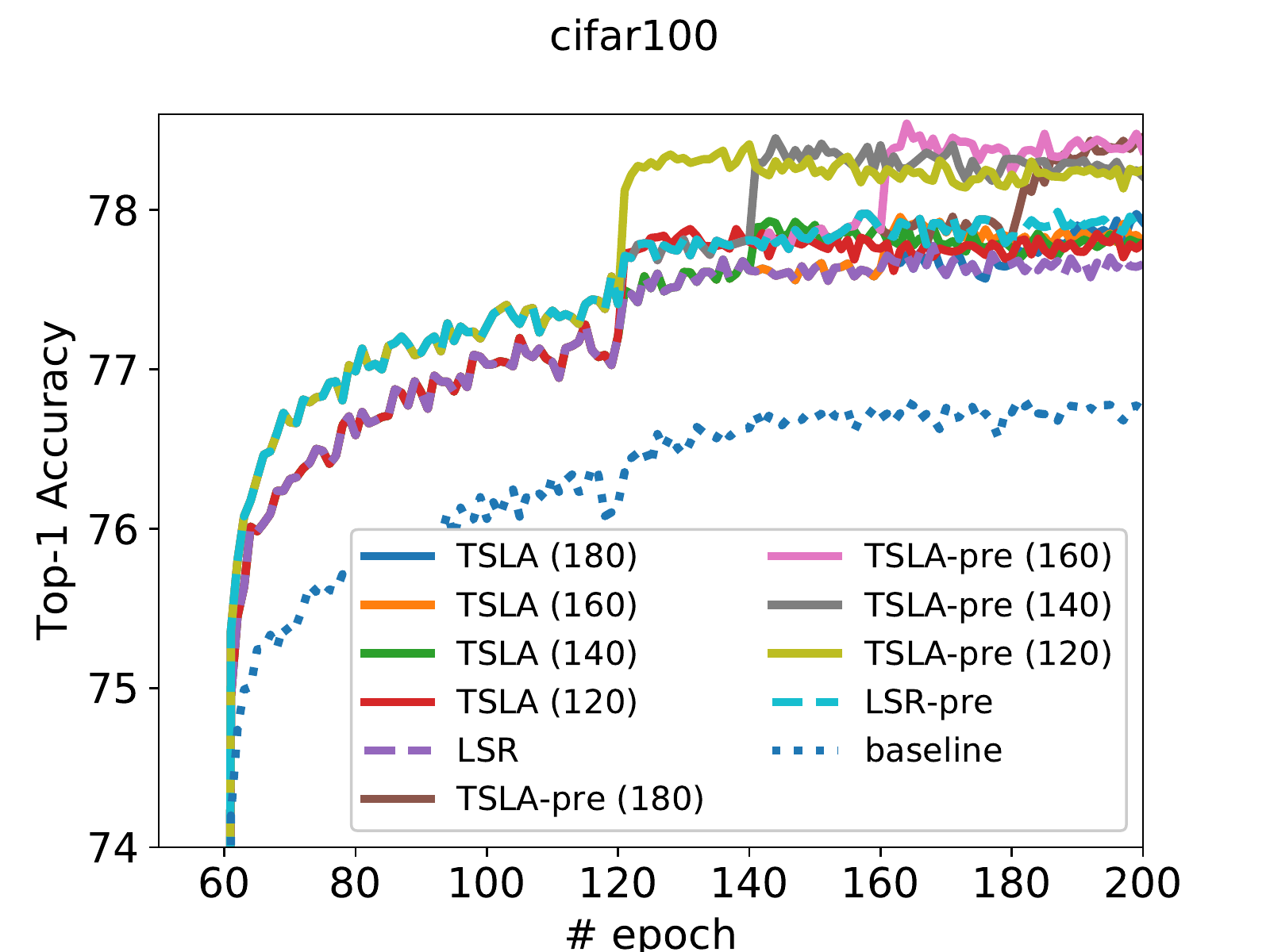}\hspace{-0.15in}
      \includegraphics[width=0.34\textwidth]{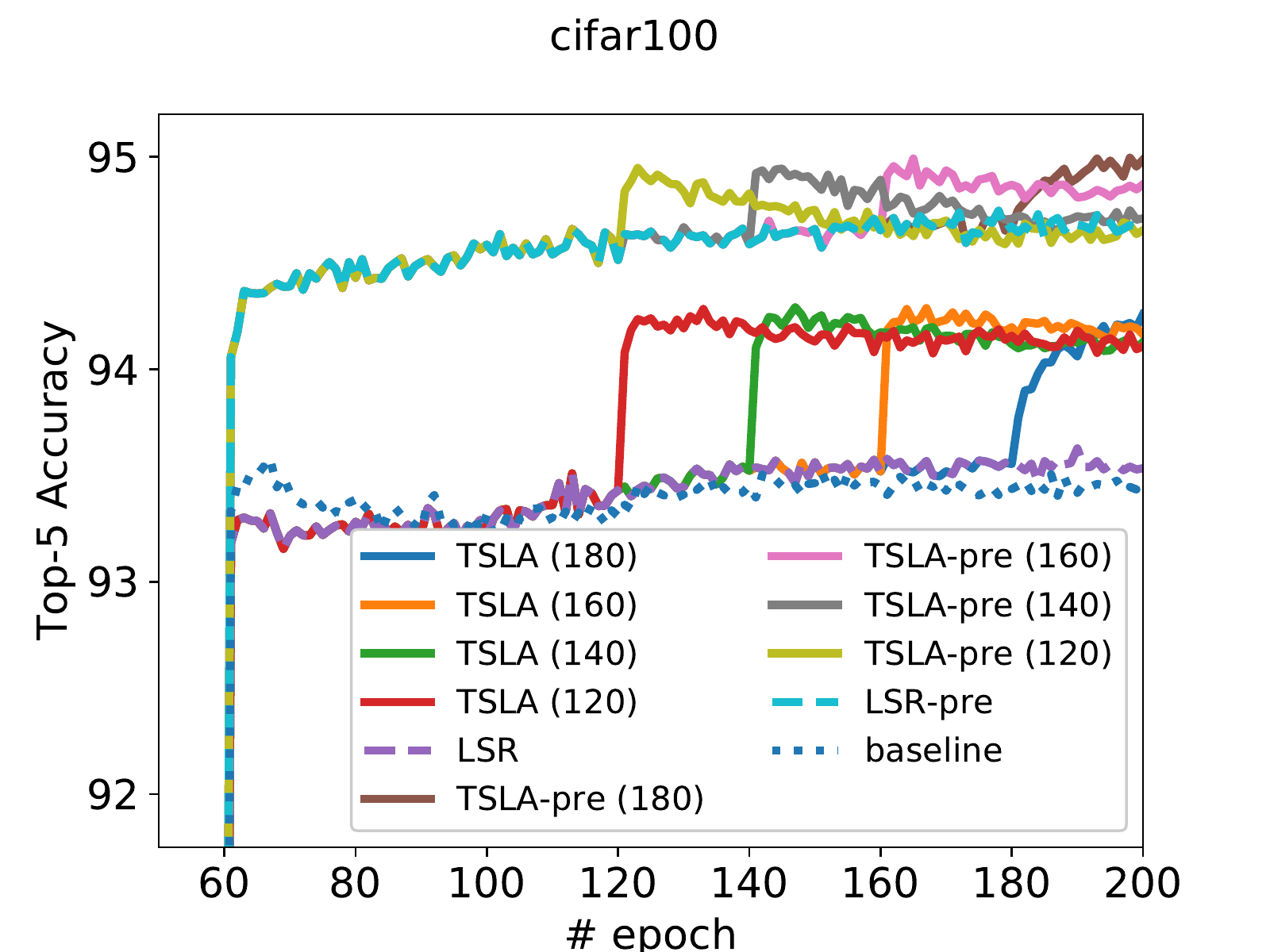}\hspace{-0.15in}
            \includegraphics[width=0.34\textwidth]{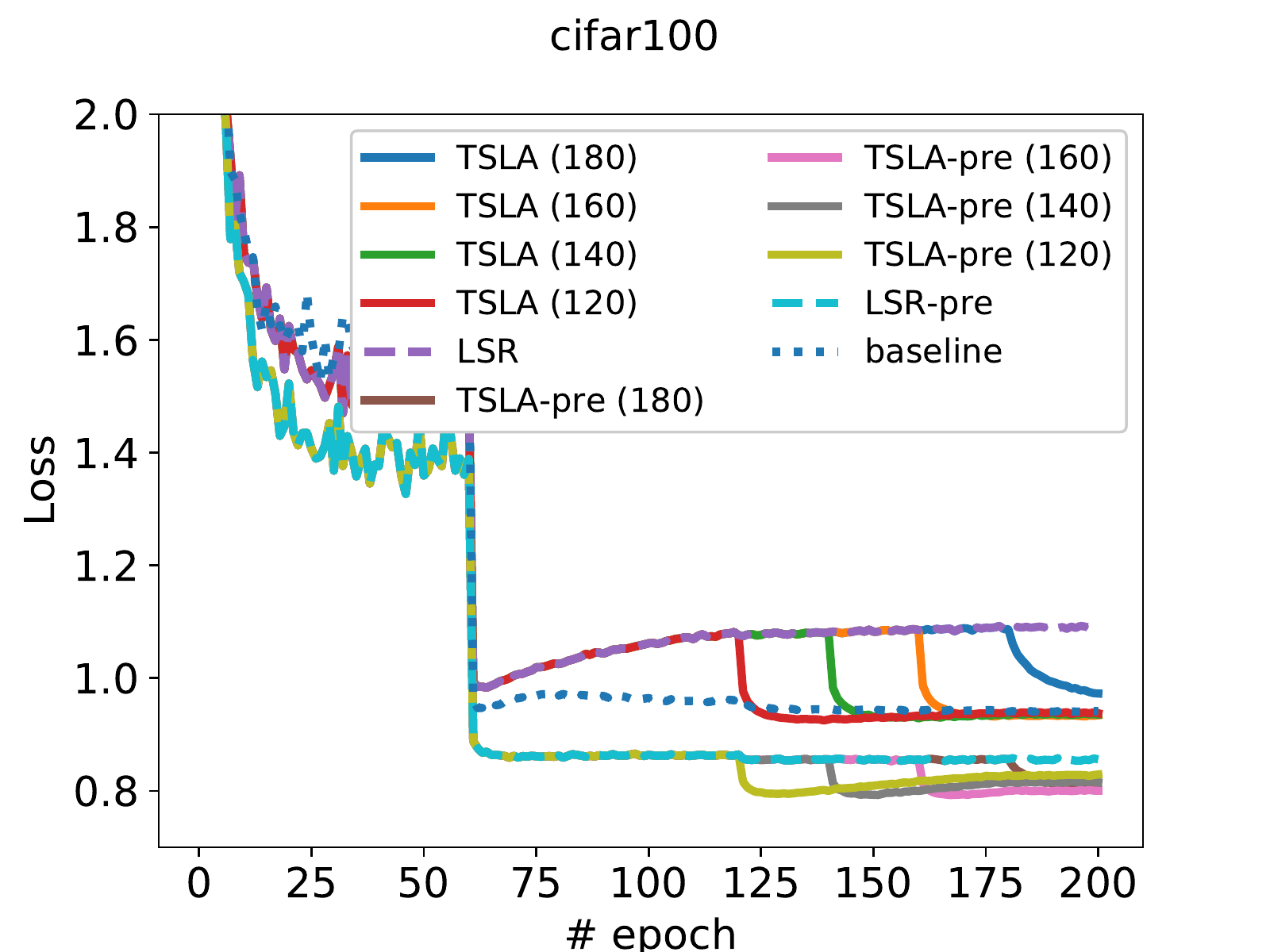}
   \caption{Testing Top-1, Top-5 Accuracy and Loss on ResNet-18 over CIFAR-100. TSLA($s$)/TSLA-pre($s$) meansTSLA/TSLA-pre drops off LSR/LSR-pre after epoch $s$.}
     \label{fig2}
\end{figure}
The total epochs of training ResNet-18~\citep{he2016deep} on CIFRA-100 is set to be 200.The weight decay with the parameter value of $5\times 10^{-4}$ is used. We use $0.1$ as the initial learning rates for all algorithms and divide them by 10 every 60 epochs suggested in~\citep{he2016deep,zagoruyko2016wide}.  For LSR and the first stage of TSLA, the value of smoothing strength $\theta$ is fixed as $\theta=0.1$, which shows the best performance for LSR. We use two different labels $\yh$ to smooth the one-hot label, the uniform distribution over all labels and the distribution predicted by an ImageNet pre-trained model which downloaded directly from PyTorch~\citep{pytorchModelArXiv}. For TSLA, we try to drop off the LSR after $s$ epochs during the training process, where $s \in\{120, 140,160,180\}$. All top-1 and top-5 accuracy on the testing data set are averaged over 5 independent random trails with their standard deviations. We summarize the results in Table~\ref{table2}, where LSR-pre and TSLA-pre indicate LSR and TSLA use the label $\yh$ based on the ImageNet pre-trained model. The results show that LSR-pre/TSLA-pre has a better performance than LSR/TSLA. The reason might be that the pre-trained model-based prediction is closer to the ground truth than the uniform prediction and it has lower variance (smaller $\delta$). Then, TSLA (LSR) with such pre-trained model-based prediction converges faster than TSLA (LSR) with uniform prediction, which verifies our theoretical findings in Sections~\ref{sec:TSLA} (Section~\ref{sec:LSR}). This observation also empirically tells us the selection of the prediction function $\yh$ used for smoothing label is the key to the success of TSLA as well as LSR.  Among all methods, the performance of TSLA-pre is the best. For top-1 accuracy, TSLA-pre(160) outperforms all other algorithms, while for top-5 accuracy, TSLA-pre(180) has the best performance. Finally, we observe from Figure~\ref{fig2} that both TSLA and TSLA-pre converge, while TSLA-pre converges to the lowest objective value. Similarly, the top-1 and op-5 accuracies show the improvements of TSLA and TSLA-pre at the point of dropping off LSR.

\section{Conclusions}
In this paper, we have studied the power of LSR in training  deep neural networks by analyzing SGD with LSR in different non-convex optimization settings. The convergence results show that an appropriate LSR with reduced label variance can help speed up the convergence. We have proposed a simple and efficient strategy so-called TSLA that can incorporate many stochastic algorithms. The basic idea of TSLA is to switch the training from smoothed label to one-hot label.
Integrating TSLA with SGD, we observe from its improved convergence result that TSLA benefits by LSR in the first stage and essentially converges faster in the second stage. Throughout extensive experiments, we have shown that TSLA improves the generalization accuracy of deep models on benchmark data sets.

\section*{Acknowledgements}
We would like to thank Jiasheng Tang and Zhuoning Yuan for several helpful discussions and comments.

\bibliographystyle{plainnat} 
\bibliography{ref}

\newpage
\appendix
\section{Technical Lemma}
Recall that the optimization problem is 
\begin{align}\label{app:opt:prob}
   \min_{\w\in\W} F(\w) := \E_{
   (\x,\y)}\left[\ell(\y, f(\w;\x)) \right],
\end{align}
where the cross-entropy loss function $\ell$ is given by
\begin{align}\label{app:loss:CE}
   \ell(\y,f(\w;\x)) = \sum_{i=1}^{K} -y_i\log\left(\frac{\exp(f_i(\w;\x))}{\sum_{j=1}^{K}\exp(f_j(\w;\x))}\right).
\end{align}
If we set
\begin{align}\label{app:soft:max}
    p(\w;\x) =(p_1(\w;\x), \dots, p_K(\w;\x))\in\R^K, \quad
    p_i(\w;\x) =-\log\left(\frac{\exp(f_i(\w;\x))}{\sum_{j=1}^{K}\exp(f_j(\w;\x))}\right),
\end{align}
the problem (\ref{app:opt:prob}) becomes
\begin{align}\label{app:opt:CE}
   \min_{\w\in\W} F(\w) := \E_{
   (\x,\y)}\left[\langle \y, p(\w;\x)\rangle \right].
\end{align}
Then the stochastic gradient with respective to $\w$ is
\begin{align}\label{app:stoc:grad}
    \nabla \ell(\y, f(\w;\x)) = \langle \y, \nabla p(\w;\x)\rangle.
\end{align}

\begin{lemma}\label{lem:var:1}
Under Assumption~\ref{ass:2} (i), we have
\begin{align*}
\E\left[ \left\| \nabla \ell(\yT_t, f(\w_t;\x_t)) - \nabla F(\w_t) \right\|^2\right] \le (1-\theta) \sigma^2 +\theta \delta \sigma^2.
\end{align*}
\end{lemma}
\begin{proof}
By the facts of $\yT_t = (1-\theta)\y_t + \theta \yh_t$ and the equation in (\ref{app:stoc:grad}), we have
\begin{align*}
\nabla \ell(\yT_t, f(\w_t;\x_t))  = (1-\theta)\nabla \ell(\y_t, f(\w_t;\x_t)) + \theta \nabla \ell(\yh_t, f(\w_t;\x_t)).
\end{align*}
Therefore, 
\begin{align*}
\nonumber& \E\left[ \left\|  \nabla \ell(\yT_t, f(\w_t;\x_t)) - \nabla F(\w_t) \right\|^2\right] \\
\nonumber= & \E\left[ \left\| (1-\theta)[\nabla \ell(\y_t, f(\w_t;\x_t)) - \nabla F(\w_t)] + \theta [ \nabla \ell(\yh_t, f(\w_t;\x_t)) - \nabla F(\w_t)]\right\|^2\right]\\
\nonumber \overset{(a)}{\le}  & (1-\theta)\E\left[ \left\|  \nabla \ell(\y_t, f(\w_t;\x_t)) - \nabla F(\w_t)\right\|^2\right] + \theta \E\left[ \left\| \nabla \ell(\yh_t, f(\w_t;\x_t)) - \nabla F(\w_t)\right\|^2\right]\\
\overset{(b)}{\le}  & (1-\theta) \sigma^2 +\theta \delta \sigma^2,
\end{align*}
where (a) uses the convexity of norm, i.e., $\|(1-\theta)\a + \theta \b\|^2\le (1-\theta)\|\a \|^2 + \theta\|\b\|^2$; (b) uses assumption~\ref{ass:2} (i) and the definitions in (\ref{variance:output:sm:label}), and Assumption~\ref{ass:2} (i).
\end{proof}

\section{Proof of Theorem~\ref{thm:lsr}}\label{app:thm:lsr}
\begin{proof}
By the smoothness of objective function $F(\w)$ in Assumption~\ref{ass:2} (ii) and its remark, we have
\begin{align}\label{thm2:ineq:1}
\nonumber& F(\w_{t+1}) - F(\w_t)\\
\nonumber\le& \langle \nabla F(\w_t), \w_{t+1}-\w_t\rangle  + \frac{L}{2}\| \w_{t+1}-\w_t\|^2 \\
\nonumber\overset{(a)}{=}& -\eta \left\langle \nabla F(\w_t), \nabla \ell(\yT_t, f(\w_t;\x_t)) \right\rangle + \frac{\eta^2 L}{2}\left\|  \nabla \ell(\yT_t, f(\w_t;\x_t))\right\|^2 \\
\nonumber\overset{(b)}{=}& -\frac{\eta}{2} \left\| \nabla F(\w_t)\right\|^2 + \frac{\eta}{2}\left\|  \nabla F(\w_t) - \nabla \ell(\yT_t, f(\w_t;\x_t))\right\|^2 + \frac{\eta(\eta L-1)}{2}\left\|  \nabla \ell(\yT_t, f(\w_t;\x_t))\right\|^2\\
\le& -\frac{\eta}{2} \left\| \nabla F(\w_t)\right\|^2 + \frac{\eta}{2}\left\|  \nabla F(\w_t) - \nabla \ell(\yT_t, f(\w_t;\x_t))\right\|^2,
\end{align}
where (a) is due to the update of $\w_{t+1}$; (b) is due to $\langle \a,-\b\rangle = \frac{1}{2}\left( \|\a-\b\|^2 - \|\a\|^2 - \|\b\|^2\right)$; (c) is due to $\eta = \frac{1}{L}$. Taking the expectation over $(\x_t,\yT_t)$ on the both sides of (\ref{thm2:ineq:1}), we have
\begin{align}\label{thm2:ineq:2}
\nonumber& \E\left[F(\w_{t+1}) - F(\w_t)\right]\\
\nonumber\le& -\frac{\eta}{2} \E\left[\left\| \nabla F(\w_t)\right\|^2\right] + \frac{\eta}{2}\E\left[\left\|  \nabla F(\w_t) - \nabla \ell(\yT_t, f(\w_t;\x_t))\right\|^2\right]\\
\le & -\frac{\eta}{2} \E\left[\| \nabla F(\w_t)\|^2 \right] + \frac{\eta}{2} \left((1-\theta) \sigma^2 +\theta \delta \sigma^2\right).
\end{align}
where the last inequality is due to Lemma~\ref{lem:var:1}.
Then inequality (\ref{thm2:ineq:2}) implies
\begin{align*}
\nonumber \frac{1}{T}\sum_{t=0}^{T-1}\E\left[\|\nabla F(\w_t)\|^2\right]\le  &\frac{2F(\w_{0}) }{\eta T} + (1-\theta) \sigma^2 +\theta \delta \sigma^2\\
\nonumber \overset{(a)}{=} &\frac{2F(\w_{0}) }{\eta T} +  \frac{2\delta}{1+\delta}\sigma^2\\
\overset{(b)}{\le} & \frac{2F(\w_{0}) }{\eta T} + 2\delta \sigma^2,
\end{align*}
where (a) is due to $\theta = \frac{1}{1+\delta}$; (b) is due to $\frac{1}{1+\delta}\le 1$.
\end{proof}

\section{Convergence Analysis of SGD without LSR ($\theta=0$)}\label{supp:baseline}
\begin{thm}\label{thm:0}
Under Assumption \ref{ass:2}, the solutions $w_t$ from Algorithm~\ref{alg:lsr}  with $\theta=0$ satisfy
\begin{align*}
\frac{1}{T}\sum_{t=0}^{T-1}\E\left[\|\nabla F(\w_t)\|^2\right]\le  \frac{2F(\w_{0}) }{\eta T} + \eta L \sigma^2.
\end{align*}
In order to have $\E_R[\|\nabla F(\w_R)\|^2]\le\epsilon^2$, it suffices to set $\eta = \min\left(\frac{1}{L}, \frac{\epsilon^2}{2L\sigma^2}\right)$ and $T = \frac{4F(\w_0)}{\eta\epsilon^2}$, the total complexity is $O\left( \frac{1}{\epsilon^4}\right)$.
\end{thm} 
\begin{proof}
By the smoothness of objective function $F(\w)$ in Assumption~\ref{ass:2} (ii) and its remark, we have
\begin{align}\label{thm0:ineq:1}
\nonumber& F(\w_{t+1}) - F(\w_t)\\
\nonumber\le& \langle \nabla F(\w_t), \w_{t+1}-\w_t\rangle  + \frac{L}{2}\| \w_{t+1}-\w_t\|^2 \\
\overset{(a)}{=}& -\eta \left\langle \nabla F(\w_t), \nabla \ell(\y_t, f(\w_t;\x_t)) \right\rangle + \frac{\eta^2 L}{2}\left\|  \nabla \ell(\y_t, f(\w_t;\x_t))\right\|^2,
\end{align}
where (a) is due to the update of $\w_{t+1}$. Taking the expectation over $(\x_t;\y_t)$ on the both sides of (\ref{thm0:ineq:1}), we have
\begin{align}\label{thm:0:ineq:2}
\nonumber &\E\left[F(\w_{t+1}) - F(\w_t) \right] \\
\nonumber \overset{(a)}{\le}& -\eta \E\left[\| \nabla F(\w_t)\|^2 \right] + \frac{\eta^2 L}{2}\E\left[\left\| \nabla \ell(\y_t, f(\w_t;\x_t)) - \nabla F(\w_t)+\nabla F(\w_t)\right\|^2\right]\\
\nonumber \overset{(b)}{=} & -\eta\E\left[\|\nabla F(\w_t)\|^2\right] +\frac{\eta^2 L}{2}\E\left[\left\| \nabla \ell(\y_t, f(\w_t;\x_t)) - \nabla F(\w_t)\right\|^2\right] + \frac{\eta^2 L}{2}\E\left[\left\|\nabla F(\w_t)\right\|^2\right]\\
\overset{(c)}{\le} & -\frac{\eta}{2}\E\left[\|\nabla F(\w_t)\|^2\right]+ \frac{\eta^2 L}{2} \sigma^2 . 
\end{align}
where (a) and (b) use Assumption~\ref{ass:2} (i); (c) uses the facts that $\eta\le\frac{1}{L}$ and Assumption~\ref{ass:2} (i).
The inequality (\ref{thm:0:ineq:2}) implies
\begin{align*}
\frac{1}{T}\sum_{t=0}^{T-1}\E\left[\|\nabla F(\w_t)\|^2\right]\le  \frac{2F(\w_{0}) }{\eta T} + \eta L \sigma^2.
\end{align*}
By setting $\eta \le \frac{\epsilon^2}{2L\sigma^2}$ and $T = \frac{4F(\w_0)}{\eta\epsilon^2}$, we have $\frac{1}{T}\sum_{t=0}^{T-1}\E\left[\|\nabla F(\w_t)\|^2\right]\le\epsilon^2$. Thus the total complexity is in the order of $O\left( \frac{1}{\eta \epsilon^2}\right) = O\left( \frac{1}{\epsilon^4}\right)$.
\end{proof}

\section{Proof of Theorem~\ref{thm:drop}}\label{app:thm:drop}
\begin{proof}
Following the similar analysis of inequality (\ref{thm2:ineq:2}) from the proof of Theorem~\ref{thm:lsr}, we have
\begin{align}\label{thm5:ineq:1}
\nonumber &\E\left[F(\w_{t+1}) - F(\w_t) \right] \\
\le & -\frac{\eta_1}{2} \E\left[\| \nabla F(\w_t)\|^2 \right] + \frac{\eta_1}{2} \left((1-\theta) \sigma^2 +\theta \delta \sigma^2\right).
\end{align}
Using the condition in Assumption~\ref{ass:3} we can simplify the inequality from (\ref{thm5:ineq:1}) as
\begin{align*}
&\E\left[F(\w_{t+1}) -F_*\right] \\
\le & (1-\eta_1\mu)\E\left[F(\w_{t}) -F_* \right]+ \frac{\eta_1}{2} \left((1-\theta) \sigma^2 +\theta \delta \sigma^2\right)\\
\le &  \left(1 - \eta_1\mu \right)^{t+1}\E\left[F(\w_0) -F_*\right]  + \frac{\eta_1}{2} \left((1-\theta) \sigma^2 +\theta \delta \sigma^2\right)\sum_{i=0}^{t} \left(1 - \eta_1\mu/2\right)^i\\
\le &  \left(1 - \eta_1\mu \right)^{t+1}\E\left[F(\w_0)\right]  + \frac{\eta_1}{2} \left((1-\theta) \sigma^2 +\theta \delta \sigma^2\right)\sum_{i=0}^{t} \left(1 - \eta_1\mu/2\right)^i,
\end{align*}
where the last inequality is due to the definition of loss function that $F_*\ge 0$.
Since $\eta_1 \le \frac{1}{L} < \frac{1}{\mu}$, then $\left(1 - \eta_1\mu \right)^{t+1} < \exp(-\eta_1\mu(t+1))$ and $ \sum_{i=0}^{t} \left(1 - \eta_1\mu\right)^i \le \frac{1}{\eta_1\mu}$. 
As a result, for any $T_1$, we have
\begin{align}\label{thm5:stage1:ineq1}
    \E\left[F(\w_{T_1}) -F_*\right] \leq \exp(-\eta_1\mu T_1)F(\w_0) +\frac{1}{2\mu} \left((1-\theta) \sigma^2 +\theta\delta \sigma^2\right).
\end{align}
Let $\theta = \frac{1}{1+\delta}$ and 
$\widehat\sigma^2 := (1-\theta) \sigma^2 +\theta \delta \sigma^2 = \frac{2\delta}{1+\delta} \sigma^2$
then $\frac{1}{2\mu} \left((1-\theta) \sigma^2 +\theta\delta \sigma^2\right)\le F(\w_0) $ since $\delta$ is small enough and $\eta_1 L \le 1$. By setting
\begin{align*}
T_1= \log\left( \frac{2\mu F(\w_0)}{\widehat\sigma^2 }\right)/(\eta_1\mu)
\end{align*}
we have
\begin{align}\label{thm5:ineq:3}
    \E\left[F(\w_{T_1}) -F_*\right] \le \frac{\widehat\sigma^2}{\mu} \le \frac{2\delta \sigma^2}{\mu}. 
\end{align}
After $T_1$ iterations, we drop off the label smoothing, i.e. $\theta=0$, then we know for any $t\ge T_1$, 
following the inequality (\ref{thm:0:ineq:2}) from the proof of Theorem~\ref{thm:0}, we have
\begin{align*}
\nonumber \E\left[F(\w_{t+1}) - F(\w_t) \right] 
\le & -\frac{\eta_2}{2}\E\left[\|\nabla F(\w_t)\|^2\right]+ \frac{\eta_2^2L\sigma^2}{2}.
\end{align*}
Therefore, we get
\begin{align}\label{thm5:ineq:4}
\nonumber\frac{1}{T_2}\sum_{t=T_1}^{T_1+T_2-1}\E\left[\|\nabla F(\w_t)\|^2\right]
\le &  \frac{2}{\eta_2 T_2}\E\left[F(\w_{T_1}) - F(\w_{T_1+T_2-1})\right] + \eta_2 L\sigma^2\\
\nonumber\overset{(a)}{\le} &  \frac{2}{\eta_2 T_2}\E\left[F(\w_{T_1}) - F_*\right] + \eta_2 L\sigma^2\\
\overset{(\ref{thm5:ineq:3})}{\le} & \frac{4 \delta \sigma^2 }{\mu \eta_2 T_2}+ \eta_2 L\sigma^2,
\end{align}
where (a) is due to $F(\w_{T_1+T_2-1})\ge F_*$.
By setting $\eta_2 = \frac{\epsilon^2}{2L\sigma^2}$ and $T_2 = \frac{8 \delta \sigma^2 }{\mu \eta_2\epsilon^2}$, we have $$\frac{1}{T_2}\sum_{t=T_1}^{T_1+T_2-1}\E\left[\|\nabla F(\w_t)\|^2\right] \le\epsilon^2.$$
\end{proof}

\end{document}